\documentclass[12pt]{colt2014} 

\title[An Inequality with Applications...]{An Inequality with Applications to Structured Sparsity and Multitask
Dictionary Learning}
\newcommand{\EE}{{{\mathbb E}}}
\newcommand{\ext}{{{\rm ext}}}
\newcommand{\tr}{{{\rm tr}}}
\newcommand{\HH}{{\cal H}}
\usepackage{amsfonts}
\input{tcilatex}

 \coltauthor{\Name{Andreas Maurer} \Email{am@andreas-maurer.eu}\\
 \addr Adalbertstrasse 55, D-80799 Munchen, Germany
 \AND
 \Name{Massimiliano Pontil} \Email{m.pontil@cs.ucl.ac.uk}\\
 \addr Department of Computer Science \\
 Centre for Computational Statistics and Machine Learning \\
 University College London, UK
  \AND
 \Name{Bernardino Romera-Paredes} \Email{bernardino.paredes.09@ucl.ac.uk}\\
 \addr Department of Computer Science and UCL Interactive Centre\\ University College London, UK
 }

\begin{document}

\maketitle

\begin{abstract}
From concentration inequalities for the suprema of Gaussian or Rademacher
processes an inequality is derived. It is applied to sharpen existing and to
derive novel bounds on the empirical Rademacher complexities of unit balls
in various norms appearing in the context of structured sparsity and
multitask dictionary learning or matrix factorization. A key role is played
by the largest eigenvalue of the data covariance matrix.
\end{abstract}

\begin{keywords}
Concentration inequalities, multitask learning, Rademacher complexity, risk bounds, structured sparsity.
\end{keywords}

\section{Introduction}
The method of Rademacher complexities \citep{Bartlett 2002,Koltchinskii
2002} has become a standard tool to prove generalization guarantees for
learning algorithms. One considers a loss class $\tciFourier $ of functions $%
f:\mathcal{X}\rightarrow 
\mathbb{R}
$, where $\mathcal{X}$ is some space of examples (such as input-output
pairs), a sample $\mathbf{x}=\left( x_{1},\dots,x_{n}\right) \in \mathcal{X}%
^{n}$ of observations and a vector $\mathbf{\epsilon }=\left( \epsilon
_{1},\dots,\epsilon _{n}\right) $ of independent Rademacher variables $%
\epsilon _{i}$ uniformly distributed on $\left\{ -1,1\right\} $. The
Rademacher complexity $\mathcal{R}( \tciFourier ,\mathbf{x}%
) $ is then defined as 
\begin{equation}
\mathcal{R}\left( \tciFourier ,\mathbf{x}\right) =\frac{2}{n}\mathbb{E}%
\sup_{f\in \tciFourier }\sum_{i=1}^{n}\epsilon _{i}f\left( x_{i}\right) .
\label{Definition Rademacher average}
\end{equation}%
Bounds on Rademacher complexities are useful in learning theory because they
lead to uniform bounds, as for example in the following result \citep
{Bartlett 2002}.

\begin{theorem}
\label{Theorem Rademacher bound}Suppose the members of $\tciFourier $ take
values in $\left[ 0,1\right]$, let $X,X_{1},\dots,X_{n}$ be iid random
variables with values in $\mathcal{X}$, and let $\mathbf{X}=(
X_{1},\dots,X_{n}) $. Then for $\delta >0$ with probability at least $%
1-\delta $ we have for every $f\in \tciFourier $ that%
\begin{equation*}
\mathbb{E}f\left( X\right) \leq \frac{1}{n}\sum_{i=1}^n f\left( X_{i}\right) +%
\mathcal{R}\left( \tciFourier ,\mathbf{X}\right) +\sqrt{\frac{9\ln 2/\delta 
}{2n}}.
\end{equation*}
\end{theorem}

Since also for any real $L$-Lipschitz function $\phi $ we have $\mathcal{R}%
\left( \phi \circ \tciFourier ,\mathbf{x}\right) \leq L~\mathcal{R}\left(
\tciFourier ,\mathbf{x}\right) $ \citep[see e.g.][]{Bartlett 2002} the utility of
Rademacher complexities is not limited to functions with values in $\left[
0,1\right] $.

For many function classes $\tciFourier $ considered in machine learning one
can find other function classes $\tciFourier _{1},\dots,\tciFourier _{M}$ such
that 
\begin{equation}
\sup_{f\in \tciFourier }\sum_{i=1}^n\epsilon _{i}f\left( x_{i}\right) \leq
\max_{m=1}^M\sup_{f\in \tciFourier _{m}}\sum_{i=1}^n\epsilon _{i}f\left(
x_{i}\right) .  \label{Decomposition Property}
\end{equation}%
Multiple kernel learning \citep[see e.g.][]{bach,Cortes 2010,Yiming}
provides an example. Let $\HH=\HH_{1}\oplus \cdots \oplus \HH_{M}$ be the direct sum
of Hilbert spaces $\HH_{m}$ with norm $\left\Vert \cdot\right\Vert$ and inner
product $\left\langle \cdot,\cdot\right\rangle$. The $\HH_{m}$ are reproducing kernel
Hilbert spaces induced by kernels $\kappa _{m}$ with
corresponding feature maps $\psi _{m}:\mathcal{X}\rightarrow
\HH_{m}$. We denote by $\psi =\left( \psi _{1},\dots,\psi
_{M}\right) :\mathcal{X}\rightarrow \HH$ the composite feature map and define the group norm for $\beta
=\left( \beta _{1},\dots,\beta _{M}\right) \in \HH$ by%
\begin{equation*}
\left\Vert \beta \right\Vert _{G}=\sum_{m=1}^{M}\left\Vert \beta
_{m}\right\Vert .
\end{equation*}%
We are interested in the class of functions $\mathcal{\tciFourier }=\left\{
x\in \mathcal{X}\mapsto \left\langle \psi \left( x\right) ,\beta
\right\rangle :\left\Vert \beta \right\Vert _{G}\leq 1\right\} $. It is easy
to see that the dual norm to $\left\Vert \cdot\right\Vert _{G}$ is $\left\Vert
z\right\Vert _{G,\ast }=\max_{m}\left\Vert z_{m}\right\Vert $. We therefore
have, writing $\tciFourier _{m}=\left\{ x\mapsto \left\langle \psi
_{m}\left( x\right) ,\beta \right\rangle _{m}:\beta \in \HH_{m}\text{, }%
\left\Vert \beta \right\Vert \leq 1\right\} $,%
\begin{equation*}
\sup_{f\in \tciFourier }\sum_{i=1}^n\epsilon _{i}f\left( x_{i}\right)
=\left\Vert \sum_{i=1}^n\epsilon _{i}\psi \left( x_{i}\right) \right\Vert
_{G,\ast }=\max_{m=1}^M\sup_{f\in \tciFourier _{m}}\sum_{i=1}^n \epsilon _{i}f\left(
x_{i}\right) ,
\end{equation*}%
as in (\ref{Decomposition Property}), so $\mathcal{R}\left( \tciFourier ,%
\mathbf{x}\right) \leq \mathcal{R}\left( \cup _{m}\tciFourier _{m},\mathbf{x}%
\right) $. In the sequel we show that many classes encountered in the study
of structured sparsity, matrix factorization and multitask dictionary
learning allow similar decompositions. 

This paper proposes a simple general method to obtain uniform bounds for
these cases and applies it to sharpen some existing ones, and to derive some new
results. The method is based on the following innocuous looking lemma.

\begin{lemma}
\label{Lemma Main}Let $M\geq 4$ and $A_{1},\dots,A_{M}\subset 
\mathbb{R}
^{n}$, $A=\cup _{m}A_{m}$, and let $\mathbf{\epsilon }=\left( \epsilon
_{1},\dots,\epsilon _{n}\right) $ be a vector of independent Rademacher
variables. Then%
\begin{equation*}
\mathbb{E}\sup_{\mathbf{z}\in A}~\left\langle \mathbf{\epsilon },\mathbf{z}%
\right\rangle \leq \max_{m=1}^M\mathbb{E}\sup_{\mathbf{z}\in A_{m}}\left\langle 
\mathbf{\epsilon },\mathbf{z}\right\rangle +4\sup_{\mathbf{z}\in
A}\left\Vert \mathbf{z}\right\Vert \sqrt{\ln M}.
\end{equation*}%
If the $\epsilon _{i}$ are replaced by standard normal variables the
same conclusion holds with the constant $4$ replaced by $2$.
\end{lemma}

For function classes $\tciFourier _{1},\dots,\tciFourier _{M}$ and a sample $%
\mathbf{x}$ let $A_{m}$ be the subset of $\mathbb{R}^{n}$ defined by $A_{m}=\left\{ \left( f\left( x_{1}\right) ,\dots,f\left(
x_{n}\right) \right) :f\in \tciFourier _{m}\right\} $ to see that the
conclusion reads

\begin{corollary}
\label{Corollary main}%
Let $\mathfrak{S}=\max\limits_{m=1}^M\mathcal{R}\left( \tciFourier _{m},\mathbf{x%
}\right)$ and let $\mathfrak{W}=\sqrt{\sup\limits_{f\in \cup \tciFourier _{m}}%
\frac{1}{n}\sum\limits_{i=1}^nf^{2}\left( x_{i}\right) }$. Then
\begin{gather*}
\mathcal{R}\left( \bigcup_{m=1}^M\tciFourier _{m},\mathbf{x}\right) \leq 
\mathfrak{S}+8\mathfrak{W}\sqrt{\frac{\ln M}{n}}.
\end{gather*}
\end{corollary}

To apply this inequality we have to bound the strong parameter $\mathfrak{S}$
and the weak parameter $\mathfrak{W}$. The real advantage of the trick lies
in the weak parameter which becomes small if the function classes have a
high linguistic specificity in the sense that individual functions are
appreciably different from zero only for rather special types of inputs. In
the context of linear prediction this corresponds to a small spectral norm
of the covariance operator, a phenomenon often associated with high
dimensionality (see Section 2.2.). In such cases the complexity of the most complex class
becomes the dominant term in the bound.

For multiple kernel learning we find with standard methods%
\begin{equation*}
\mathfrak{S}\leq \frac{2}{n}\max_{m=1}^M\sqrt{\sum_{i=1}^n\left\Vert \psi _{m}\left(
x\right) \right\Vert ^{2}}=2\max_{m=1}^M\sqrt{\frac{{\rm tr}\left( \hat{C}
\left( \psi _{m}\left( \mathbf{x}\right) \right)  \right) }{n}}
\end{equation*}%
where the uncentered empirical covariance operator of the data ${\bf z} = (z_1,\dots,z_n)$ is defined, for every vectors $v,w$, by the equation $\langle {\hat C}({\bf z})v,w \rangle = \frac{1}{n} \sum_{i=1}^n \langle v,z_i \rangle \langle z_i,w \rangle$, see also Section \ref{sec:cov} below.
The weak parameter is%
\begin{equation*}
\mathfrak{W}=\max_{m=1}^M\sqrt{\sup_{\mathbf{\beta }\in \tciFourier _{m}}\frac{1%
}{n}\sum_{i=1}^n\left\langle \beta ,\psi _{m}\left( x_{i}\right) \right\rangle
^{2}}=\max_{m=1}^M\sqrt{\lambda _{\max }\left( \hat{C}\left( \psi _{m}\left( 
\mathbf{x}\right) \right) \right) },
\end{equation*}%
where $\lambda _{\max }$ denotes the largest eigenvalue. The overall bound
is then%
\begin{equation}
\mathcal{R}\left( \tciFourier ,\mathbf{x}\right) \leq 2\max_{m=1}^M\sqrt{\frac{%
{\rm tr}\left( \hat{C}\left( \psi _{m}\left( \mathbf{x}\right) \right) \right) }{n}%
}+8\max_{m=1}^M\sqrt{\frac{\lambda _{\max }\left( \hat{C}\left( \psi _{m}\left( 
\mathbf{x}\right) \right) \right) \ln M}{n}}.  \label{Group Lasso Bound}
\end{equation}%
Note that in this example the eigenvalues of $\hat{C}( \psi _{m}( 
\mathbf{x})) $ coincide with the eigenvalues of the normalized
kernel matrix $\kappa _{m}\left( x_{i},x_{j}\right) /n$. Other authors \citep{Cortes 2010,MP2012} give a bound of order $\max\limits_{m=1}^M\sqrt{{\rm tr}( 
\hat{C}( \psi _{m}( \mathbf{x}))) \ln M/n}$.
If we divide the two bounds we see that (\ref{Group Lasso Bound}) becomes a
significant improvement when the number of kernels is large and the quotient 
$\lambda _{\max }( \hat{C}( \psi _{m}( \mathbf{x}))) /{\rm tr}( \hat{C}( \psi _{m}( \mathbf{x}))) $ is small. The latter condition will occur if the feature
representations $\psi _{m}\left( \mathbf{x}\right) $ are essentially high
dimensional, as it occurs for example with Gaussian radial basis function
kernels with small kernel width. This type of behaviour is typical of the
proposed method whose benefits become more pronounced in effectively high
dimensions. In the artificial case
of exactly spherical data $\mathbf{x}$ in $\mathbb{R}^{d}$, we even have $\lambda_{\max}( \hat{C}( \mathbf{x}))
/{\rm tr}( \hat{C}( \mathbf{x})) =1/d$ (see Section 2.2).

Of course the example of multiple kernel learning applies equally to the
group lasso \citep{YuaLin}, but Lemma \ref{Lemma Main} can also be
applied to a large class of structured sparsity norms to sharpen bounds for
overlapping groups \citep{Jacob 2009}, cone regularizers \citep{Micchelli 2011}
and the recently proposed $k$-support norm \citep{Argyriou 2012}.

Related applications give generalization guarantees for various schemes of
multitask dictionary learning or matrix factorization. As examples we
reproduce the results by \citet{MPR2013} and give novel bounds for other
matrix regularizers including multitask subspace learning. In these
applications the weak parameter is particularly important, because it is proportional to the limit of the generalization error as the number of tasks goes to infinity.

The proof of Lemma \ref{Lemma Main} relies on concentration inequalities for
the suprema of Rademacher or Gaussian processes. If the random variables $%
\epsilon _{i}$ are independent standard normal then the constant $4$ in
Lemma \ref{Lemma Main} can be replaced by $2$. On the other hand bounding
the Rademacher complexities by Gaussian complexities incurs a factor of $%
\sqrt{\pi /2}$, so little seems to be gained. We will however also give the
bound for isonormal $\mathbf{\epsilon }$ because Gaussian averages are
sometimes convenient when Slepian's inequality is applied to simplify
complicated classes.

Lemma \ref{Lemma Main} is certainly not new, although we cannot give an
exact reference. Related results appear in various disguises whenever modern
concentration inequalities are applied to empirical processes, as for
example in \citep{Ledoux 1991} or the recent book by \citet{Boucheron 2013}. We
are not aware of any reference where Lemma \ref{Lemma Main} is applied as a
systematic method to prove or improve uniform bounds as in the present
paper. The applications given are intended as illustrations of the method
and they are by no means exhaustive. The bound in Section \ref{Section
sparsity norm} has already appeared in \citep{MPR2013}, the result on
subspace learning in Section \ref{section subspace learning} is somewhat
similar to a result derived from noncommutative Bernstein inequalities in 
\citep{MP2012b}. The bounds on structured sparsity norms in Section \ref%
{Section structured sparsity} and the result for the sharing norm in Section %
\ref{section sharing norm} are new to the best of our knowledge. 

In the next section we give a proof of Lemma \ref{Lemma Main} and in Section %
\ref{Section Application examples} we give applications to structured
sparsity and dictionary learning. An appendix contains the proofs of the
concentration inequalities we use.

\section{Theory}

We provide a proof of Lemma \ref{Lemma Main} and a brief and elementary
discussion of covariances.

\subsection{The Proof of Lemma \protect\ref{Lemma Main}}

We use the following concentration inequality for the suprema of bounded or
Gaussian random processes. A proof and bibliographical remarks are provided
in the technical appendix to this paper.

\begin{theorem}
\label{Lemma concentration of norm}Let $A\subset 
\mathbb{R}
^{n}$ and let $\mathbf{\epsilon }=\left( \epsilon _{1},\dots,\epsilon
_{n}\right) $ be a vector of independent random variables satisfying $%
\left\vert \epsilon _{i}\right\vert \leq 1$. Then 
\begin{equation*}
\Pr \left\{ \sup_{\mathbf{z}\in A}~\left\langle \mathbf{\epsilon },\mathbf{z}%
\right\rangle >{{\mathbb{E}}}\sup_{\mathbf{z}\in A}~\left\langle \mathbf{%
\epsilon },\mathbf{z}\right\rangle +s\right\} \leq \exp \left( \frac{-s^{2}}{%
8\sup_{\mathbf{z}\in A}\left\Vert \mathbf{z}\right\Vert ^{2}}\right) .
\end{equation*}%
If the $\epsilon _{i}$ are replaced by standard normal variables then the
same conclusion holds and the constant $8$ can be replaced by $2$.
\end{theorem}

With Theorem \ref{Lemma concentration of norm} at hand the proof of Lemma %
\ref{Lemma Main} becomes an exercise of calculus.

\vspace{.3truecm}
\begin{proof}\textbf{of Lemma \ref{Lemma Main}.} Denote the random variable $\sup_{\mathbf{z}\in A_{m}}\left\langle \mathbf{%
\epsilon },\mathbf{z}\right\rangle $ with $F_{m}$ and set $v=\sup_{\mathbf{z}%
\in A}\left\Vert \mathbf{z}\right\Vert $. From Theorem \ref{Lemma
concentration of norm} we have for $s>0$%
\begin{equation*}
\Pr \left\{ F_{m}>\mathbb{E}F_{m}+s\right\} 
\leq e^{-s^{2}/\left( 2b^{2}v^{2}\right) },
\end{equation*}%
where $b$ is either $1$ in the Gaussian or $2$ in the bounded case. A union
bound gives%
\begin{equation}
\Pr \left\{ \max_{m}F_{m}>\max_{m}\mathbb{E}F_{m}+s\right\} \leq
Me^{-s^{2}/\left( 2b^{2}v^{2}\right) }.  \label{Concentration of max_m F_m}
\end{equation}%
We now have, for any positive $\delta $,%
\begin{eqnarray*}
\mathbb{E}\max_{m}F_{m} &\leq &\max_{m}\mathbb{E}F_{m}+\delta +\int_{\max_{m}%
\mathbb{E}F_{m}+\delta }^{\infty }\Pr \left\{ \max_{m}F_{m}>s\right\} ds \\
&=&\max_{m}\mathbb{E}F_{m}+\delta +\int_{\delta }^{\infty }\Pr \left\{
\max_{m}F_{m}>\max_{m}\mathbb{E}F_{m}+s\right\} ds \\
&\leq &\sup_{m}\mathbb{E}F_{m}+\delta +M\int_{\delta }^{\infty
}e^{-s^{2}/\left( 2b^{2}v^{2}\right) }ds.
\end{eqnarray*}%
The first step holds because probabilities do not exceed one, the second is
a change of variable and finally we used (\ref{Concentration of max_m F_m}).
By a well known approximation we can bound the integral by 
\begin{equation*}
\int_{\delta }^{\infty }e^{-s^{2}/\left( 2b^{2}v^{2}\right) }ds\leq \frac{%
b^{2}v^{2}}{\delta }e^{-\delta ^{2}/\left( 2b^{2}v^{2}\right) }.
\end{equation*}%
Using $\delta =\sqrt{2b^{2}v^{2}\ln M}$ we have%
\begin{eqnarray*}
\EE \max_{m}F_{m} &\leq &\max_{m}EF_{m}+\delta +\frac{Mb^{2}v^{2}}{\delta }%
e^{-\delta ^{2}/\left( 2b^{2}v^{2}\right) } \\
&=&\max_{m}\EE F_{m}+bv\left( \sqrt{2\ln M}+\frac{1}{\sqrt{2\ln M}}\right) \leq \max_{m}\EE F_{m}+2bv\sqrt{\ln M},
\end{eqnarray*}%
since we assumed $M\geq 4$.
\end{proof}

\subsection{Covariances}
\label{sec:cov}
Let $\mathbf{x}=\left( x_{1},\dots,x_{n}\right) $ be a sequence of points in a
finite or infinite dimensional real Hilbert space $\HH$ with inner product $%
\left\langle \cdot,\cdot\right\rangle $ and norm $\left\Vert \cdot\right\Vert $. The
(uncentered, empirical) covariance operator $\hat{C}\left( \mathbf{x}\right) 
$ is defined by 
\begin{equation*}
\left\langle \hat{C}\left( \mathbf{x}\right) v,w\right\rangle =\frac{1}{n}%
\sum_{i=1}^n\left\langle x_{i},v\right\rangle \left\langle x_{i},w\right\rangle
,~v,w\in \HH.
\end{equation*}%
$\hat{C}\left( \mathbf{x}\right) $ is positive semidefinite and of rank at
most $n$. Its trace is given by 
\begin{equation*}
{\rm tr}\left( \hat{C}\left( \mathbf{x}\right) \right) =\frac{1}{n}%
\sum_{i=1}^n\left\Vert x_{i}\right\Vert ^{2}.
\end{equation*}%
In the sequel we will frequently use the inequality%
\begin{equation}
\mathbb{E}\left\Vert \sum_{i=1}^n\epsilon _{i}x_{i}\right\Vert \leq \left( 
\mathbb{E}\left\Vert \sum_{i=1}^n\epsilon _{i}x_{i}\right\Vert ^{2}\right)
^{1/2}=\sqrt{n~{\rm tr}\left( \hat{C}\left( \mathbf{x}\right) \right) }
\label{Trace inequality}
\end{equation}%
where the $\epsilon _{i}$ are either independent Rademacher or standard
normal variables and we used Jensen's inequality and the
orthonormality properties of the $\epsilon _{i}$.

The largest eigenvalue of the covariance is 
\begin{equation*}
\lambda _{\max }\left( \hat{C}\left( \mathbf{x}\right) \right)
=\sup_{\left\Vert v\right\Vert \leq 1}\left\langle \hat{C}\left( \mathbf{x}%
\right) v,v\right\rangle =\sup_{\left\Vert v\right\Vert \leq 1}\frac{1}{n}%
\sum_{i}\left\langle x_{i},v\right\rangle ^{2}.
\end{equation*}%
Clearly the ratio $\lambda _{\max }( \hat{C}( \mathbf{x})) /{\rm tr}( \hat{C}( \mathbf{x}))$ is upper bounded by $1$ and it can 
be as small as $1/n$ for exactly spherical data.

For a practical example suppose that the inputs lie in $\mathbb{R}^{d}$ and that the Hilbert space $\HH$ is induced by a Gaussian
kernel, so that%
\begin{equation*}
\left\langle \psi \left( x\right) ,\psi \left( y\right) \right\rangle
=\kappa \left( x,y\right) =\exp \left( \frac{-\left\Vert x-y\right\Vert _{{{\mathbb R}^d}}^{2}}{\sigma ^{2}}\right) ,
\end{equation*}
where $\psi $ is the embedding feature map and $\|\cdot\|_{{\mathbb R}^d}$ is the standard inner product in ${\mathbb R}^d$. Now let a dataset $\mathbf{x}%
=\left( x_{1},\dots,x_{n}\right) $ be given with $x_{i}\in 
\mathbb{R}^{d}$. Clearly ${\rm tr}( \hat{C}( \psi ( \mathbf{x}))) =1$. Suppose that $\Delta $ is the smallest distance between any two
observations $\Delta =\min_{i\neq j}\left\Vert x_{i}-x_{j}\right\Vert _{{\mathbb R}^d}$. It is easy to see that the largest eigenvalue of the covariance is $%
1/n$ times the largest eigenvalue of the kernel matrix $K=\kappa (
x_{i},x_{j})_{i,j=1}^n$. Thus%
\begin{eqnarray*}
\lambda _{\max }\left( \hat{C}\left( \psi \left( \mathbf{x}\right) \right)
\right) &=&\frac{1}{n}\sup_{\Vert \mathbf{\alpha }\Vert_{{\mathbb R}^d} \leq 1}\left\langle K\alpha ,\alpha
\right\rangle =\frac{1}{n}+\frac{1}{n}\sup_{\Vert \mathbf{\alpha }\Vert_{{\mathbb R}^d} \leq 1}\sum_{i\neq
j}\alpha _{i}\alpha _{j}\exp\left(-\frac{\left\Vert x_{i}-x_{j}\right\Vert ^{2}_{{\mathbb R}^d}}{{\sigma
^{2}}}\right) \\
&\leq &\frac{1}{n}+\frac{e^{-\Delta ^{2}/\sigma ^{2}}}{n}\sup_{
\Vert \mathbf{\alpha }\Vert_{{\mathbb R}^d} \leq 1} \sum_{i\neq j} \left\vert\alpha _{i}\right\vert \left\vert \alpha
_{j}\right\vert \leq \frac{1}{n}+e^{-\Delta ^{2}/\sigma ^{2}}.
\end{eqnarray*}%
This is also a bound on the ratio $\lambda _{\max }( \hat{C}(
\mathbf{x})) /{\rm tr}( \hat{C}( \mathbf{x})) $%
, since the trace of the covariance is $1$ for the Gaussian kernel. It
follows that the weak parameter in our bounds decreases with the width $\sigma$ of
the kernel. Of course this is only part of the story. We hasten to add that
decreasing the kernel width will have an adverse effect on generalization. 
Nevertheless our results seem to indicate that, at least in the context of
the applications below, the kernel width can be chosen smaller than
suggested by conventional bounds, where $\lambda _{\max }$ is replaced by
the trace \citep{MP2012,Kakade,Cortes
2010}. This is particularly true for multitask learning with a large
number of tasks, where $\lambda _{\max }$ scales the limiting generalization
error, as shown below.

We state our bounds in terms of uncentered covariances, but of course they
also apply as well if the data is centered by subtracting $\mathbf{\bar{x}}%
=\left( 1/n\right) \sum_{i}x_{i}$ from each data point. It is easy to see
that $\left\langle \hat{C}\left( \mathbf{x}-\mathbf{\bar{x}}\right)
v,v\right\rangle \leq \left\langle \hat{C}\left( \mathbf{x}\right)
v,v\right\rangle $ for all $v$, so that $\tr\left( \hat{C}\left( \mathbf{x}-%
\mathbf{\bar{x}}\right) \right) \leq \tr\left( \hat{C}\left( \mathbf{x}%
\right) \right) $ and $\lambda _{\max }\left( \hat{C}\left( \mathbf{x}-%
\mathbf{\bar{x}}\right) \right) \leq \lambda _{\max }\left( \hat{C}\left( 
\mathbf{x}\right) \right) $. Our bounds can therefore only benefit from
centering. This is relevant when calculating the advantage of our bounds in
practice. With MNIST and raw pixel data without kernel, we found $\lambda
_{\max }\left( \hat{C}\left( \mathbf{x}\right) \right) /\tr\left( \hat{C}%
\left( \mathbf{x}\right) \right) \approx 0.95$ for uncentered data, but $<0.1
$ for centered data.

\section{Application Examples\label{Section Application examples}}

We use Lemma \ref{Lemma Main} to derive general bounds for a class of
structured sparsity norms. Then we discuss several applications to multitask
dictionary learning.

\subsection{Structured Sparsity\label{Section structured sparsity}}

Suppose $\HH$ is a separable, real, finite or infinite dimensional
Hilbert space with norm and inner product $\left\Vert \cdot\right\Vert $ and $%
\left\langle \cdot,\cdot\right\rangle $, and that $\mathcal{P}=\left\{
P_{1},\dots,P_{M}\right\} $ is a collection of symmetric bounded operators
whose ranges together span $\HH$. We consider the infimal convolution norm on $\HH$ 
\begin{equation*}
\left\Vert \beta \right\Vert _{\mathcal{P}}=\inf \left\{
\sum_{m=1}^{M}\left\Vert v_{m}\right\Vert :v_{m}\in \HH,\text{ }%
\sum_{m=1}^{M}P_{m}v_{m}=\beta \right\} ,\beta \in \HH,
\end{equation*}%
whose dual norm is given by%
\begin{equation*}
\left\Vert x\right\Vert _{\mathcal{P,\ast }}=\max_{m=1}^{M}\left\Vert
P_{m}x\right\Vert .
\end{equation*}%
These are the norms considered in \citep{MP2012} and include among others the
group lasso, overlapping groups and multiple kernel learning. In the case of
multiple kernel learning, for example, $P_{m}$ is just the projection onto
the $m$-th RKHS. We are interested in the Rademacher complexity of the
function class $\tciFourier =\left\{ x\in \HH\mapsto \left\langle \beta
,x\right\rangle :\left\Vert \beta \right\Vert _{\mathcal{P}}\leq 1\right\} $%
. Now%
\begin{eqnarray*}
\mathbb{E}\sup_{\left\Vert \beta \right\Vert _{\mathcal{P}}\leq
1}\left\langle \beta ,\sum_{i}\epsilon _{i}x_{i}\right\rangle &=&\mathbb{E}%
\left\Vert \sum_{i}\epsilon _{i}x_{i}\right\Vert _{\mathcal{P,\ast }}=\mathbb{%
E}\max_{m}\left\Vert \sum_{i}\epsilon _{i}P_{m}x_{i}\right\Vert \\
&=&\mathbb{E}\max_{m}\sup_{\left\Vert \beta \right\Vert =1}\sum_{i}\epsilon
_{i}\left\langle \beta ,P_{m}x_{i}\right\rangle =\mathbb{E}%
\max_{m}\sup_{f\in \tciFourier _{m}}\sum_{i}f\left( x_{i}\right) ,
\end{eqnarray*}%
where $\tciFourier _{m}=\left\{ x\in \HH\mapsto \left\langle \beta
,P_{m}x\right\rangle :\left\Vert \beta \right\Vert \leq 1\right\} ,$ so
Lemma \ref{Lemma Main} can be applied. Using (\ref{Trace inequality}) strong
and weak parameters are%
\begin{eqnarray*}
\mathfrak{S} &=&\frac{2}{n}\max_{m}\mathbb{E}\left\Vert \sum_{i}\epsilon
_{i}P_{m}x_{i}\right\Vert \leq 2\max_{m}\sqrt{\frac{{\rm tr}\left( \hat{C}\left(
P_{m}\mathbf{x}\right) \right) }{n}}, \\
\mathfrak{W} &=&\sqrt{\max_{m}\sup_{y}\frac{1}{n}\sum_{i=1}^{n}\left\langle
y,P_{m}x_{i}\right\rangle ^{2}}=\max_{m}\sqrt{\lambda _{\max }\left( \hat{C}%
\left( P_{m}\mathbf{x}\right) \right) }.
\end{eqnarray*}%
Lemma \ref{Lemma Main} yields the overall bound%
\begin{equation*}
\mathcal{R}\left( \tciFourier ,\mathbf{x}\right) \leq 2\max_{m}\sqrt{\frac{%
{\rm tr}\left( \hat{C}\left( P_{m}\mathbf{x}\right) \right) }{n}}+8\max_{m}\sqrt{%
\frac{\lambda _{\max }\left( \hat{C}\left( P_{m}\mathbf{x}\right) \right)
\ln M}{n}}
\end{equation*}%
which improves over the bounds in \citep{MP2012,Kakade,Cortes
2010}, whenever $\max_{m}\lambda _{\max }\left( \hat{C}\left( P_{m}\mathbf{x}%
\right) \right) $ is appreciably smaller than $\max_{m}{\rm tr}\left( \hat{C}%
\left( P_{m}\mathbf{x}\right) \right) $.

\subsection{Generalities on Multitask Dictionary Learning}

We first consider multitask feature learning in general \citep{Baxter 2000}.
In subsequent sections we give exemplifying bounds for three specific
regularizers.

With inputs in some space $\mathcal{X}$ and intermediate feature
representations in some feature space $\mathcal{X}^{\prime }$ let $\mathcal{G%
}$ be a class of feature maps $g:\mathcal{X}\rightarrow \mathcal{X}^{\prime
} $ and let $\tciFourier $ be a class of vector valued functions $\mathbf{f}:%
\mathcal{X}^{\prime }\rightarrow 
\mathbb{R}
^{T}$. We study the vector valued function class 
\begin{equation*}
\tciFourier \circ \mathcal{G}=\left\{ x\mapsto \left( f_{1}\left( g\left(
x\right) \right) ,\dots,f_{T}\left( g\left( x\right) \right) \right) :g\in 
\mathcal{G},\mathbf{f}\in \tciFourier \right\} \text{.}
\end{equation*}%
Now let $x_{ti}\in \mathcal{X}$ be the $i$-th example available for the $t$%
-th task, $1\leq i\leq n$ and $1\leq t\leq T$. The multitask Rademacher
average is now%
\begin{equation*}
\mathcal{R}\left( \tciFourier \circ \mathcal{G},\mathbf{x}\right) =\frac{2}{%
nT}\mathbb{E}\sup_{\mathbf{f}\in \tciFourier }\sup_{g\in \mathcal{G}%
}\sum_{t=1}^T \sum_{i=1}^n \epsilon _{ti} f_{t}\left( g\left( x_{ti}\right) \right) ,
\end{equation*}%
where the $\epsilon _{ti}$ are $nT$ independent Rademacher variables. The
purpose of bounding these averages is to obtain uniform bounds in $%
\tciFourier \circ \mathcal{G}$ on the average multitask error $\frac{1}{T}%
\sum_{t}\mathbb{E}f_{t}\left( g\left( x_{t}\right) \right) $, in terms of
its empirical counterpart when $x_{t}$ is sampled iid to $x_{ti}$ \citep[see e.g.][]{Zhang 2005}. If $\tciFourier $ is finite then the above expression
evidently has the form required for application of Lemma \ref{Lemma Main},
which gives%
\begin{equation*}
\mathcal{R}\left( \tciFourier \circ \mathcal{G},\mathbf{x}\right) \leq 
\mathfrak{S}+8\mathfrak{W}\sqrt{\frac{\ln \left( \left\vert \mathbf{%
\tciFourier }\right\vert \right) }{nT}}
\end{equation*}%
with strong and weak parameters%
\begin{equation*}
\mathfrak{S}=\frac{2}{nT}\max_{\mathbf{f}\in \tciFourier }\mathbb{E}%
\sup_{g\in \mathcal{G}}\sum_{t,i}\epsilon _{ti}f_{t}\left( g\left(
x_{ti}\right) \right) \text{ and }\mathfrak{W}=\sqrt{\max_{\mathbf{f}\in
\tciFourier }\sup_{g\in \mathcal{G}}\frac{1}{nT}\sum_{t,i}f_{t}\left( g\left(
x_{ti}\right) \right) ^{2}}.
\end{equation*}%
In some cases the vector valued functions in $\tciFourier $ consist of
unconstrained $T$-tuples of real valued functions chosen from some class $%
\tciFourier _{0}$ independently for each task, so that $\tciFourier =\left(
\tciFourier _{0}\right) ^{T}$. In this case $\ln \left\vert \mathbf{%
\tciFourier }\right\vert =T\ln \left\vert \tciFourier _{0}\right\vert $, so
the above bound becomes 
\begin{equation*}
\mathcal{R}\left( \tciFourier \circ \mathcal{G},\mathbf{x}\right) \leq 
\mathfrak{S}+8\mathfrak{W}\sqrt{\frac{\ln \left\vert \mathbf{\tciFourier }%
_{0}\right\vert }{n}}.
\end{equation*}%
Typically $\mathfrak{S}\rightarrow 0$ in the multitask limit $T\rightarrow
\infty $. This highlights the role of the weak parameter $\mathfrak{W}$. It
controls what is left over of the generalization error for fixed $n$ if $T$
is large.

For a more concrete setting let $\mathcal{X}=\HH$ be a Hilbert space and for
some fixed $K\in 
\mathbb{N}
$ let $\mathcal{D}$ be the set of all dictionaries $D=\left(
d_{1},\dots,d_{K}\right) \in \HH^{K}$ satisfying $\left\Vert d_{k}\right\Vert
\leq 1$ for each $k$. The intermediate representation space will now be $%
\mathcal{X}^{\prime }=%
\mathbb{R}
^{K}$ and the admissible feature maps are 
\begin{equation*}
\mathcal{G}=\left\{ x\in \HH\mapsto \left( \left\langle d_{1},x\right\rangle
,\dots,\left\langle d_{K},x\right\rangle \right) :\left(
d_{1},\dots,d_{K}\right) \in \mathcal{D}\right\} .
\end{equation*}%
For a compact set of matrices $\mathcal{W}\subset 
\mathbb{R}
^{T \times K}$ we define the class $\tciFourier \left( \mathcal{W}\right) $ as 
\begin{equation*}
\tciFourier \left( \mathcal{W}\right) =\left\{ y\in 
\mathbb{R}
^{K}\mapsto \left( \sum_{k}W_{1k}y_{k},\dots,\sum_{k}W_{Tk}y_{k}\right) :W\in 
\mathcal{W}\right\} \text{.}
\end{equation*}%
For fixed dictionary $D=\left( d_{1},\dots,d_{K}\right) $ and fixed $\mathbf{%
\epsilon }$ the expression $\sum_{t,i}\epsilon
_{ti}\sum_{k}W_{tk}\left\langle d_{k},x_{ti}\right\rangle $ is linear in $W$
and therefore attains its maximum at an extreme point $W^{\ast }\in \ext\left( 
\mathcal{W}\right) $. Thus $\mathcal{R}\left( \tciFourier \left( \mathcal{W}%
\right) \circ \mathcal{G},\mathbf{x}\right) =\mathcal{R}\left( \tciFourier
\left( \ext\left( \mathcal{W}\right) \right) \circ \mathcal{G},\mathbf{x}%
\right) $. But the set of extreme points $\ext\left( \mathcal{W}\right) $ is
often finite in which case our method can be applied. In the sequel we give
two examples. Another possiblity is that $\mathcal{W}$ has a reasonable
finite approximation, for which we will also give an example.

\subsection{Dictionary Learning with the Sparsity Norm\label{Section
sparsity norm}}

For matrices $W\in 
\mathbb{R}
^{T \times K}$ we define the sparsity norm\footnote{%
Some authors \citep[see e.g.][]{Kakade} would call this the $1/\infty $-, others
\citep[see e.g.][]{Negahban 2008} the $\infty /1$-norm, depending on preference
for either computational or typographical order. To avoid confusion we use
the wedge $\wedge $ and refer to it as the ``sparsity norm".}%
\begin{equation*}
\left\Vert W\right\Vert _{\wedge }:=\max_{t=1}^T\sum_{k}\left\vert
W_{tk}\right\vert
\end{equation*}%
and consider the class of matrices $\mathcal{W}_{\wedge }\mathcal{=}\left\{
W\in 
\mathbb{R}
^{T \times K}:\left\Vert W\right\Vert _{\wedge }\leq 1\right\} $. Observe that $%
\tciFourier \left( \mathcal{W}_{\wedge }\right) =\left( \tciFourier
_{\rm Lasso}\right) ^{T}$, where $\tciFourier _{\rm Lasso}$ is the class given by
linear functionals on $%
\mathbb{R}
^{K}$ with $\ell _{1}$-norm bounded by $1$. One checks that the set of
extreme points is%
\begin{equation*}
\ext\left( \mathcal{W}_{\wedge }\right) =\left\{ W:W_{tk}=\sigma _{t}\delta
_{\phi _{t},k},\mathbf{\sigma }\in \left\{ -1,1\right\} ^{T}\text{, }\phi
\in \left\{ 1,\dots,K\right\} ^{T}\right\} ,
\end{equation*}%
where $\delta $ is the Kronecker delta. In words: $W$ is an extreme point iff
for each $t$ there is only one nonzero $W_{t\phi _{t}}\in \left\{
-1,1\right\} $, all the other $W_{tk}$ being zero. Now $\ext\left( \mathcal{W}%
_{\wedge }\right) $ is finite with cardinality $\left\vert \ext\left( 
\mathcal{W}_{\wedge }\right) \right\vert =\left( 2K\right) ^{T}$, so our
method is applicable to give bounds for the class $\tciFourier \left( 
\mathcal{W}_{\wedge }\right) \circ \mathcal{G}$. We bound the strong
parameter as 
\begin{eqnarray*}
\mathfrak{S} &=&\frac{2}{nT}\max_{W\in \ext\left( \mathcal{W}_{\wedge
}\right) }\mathbb{E}\sup_{D\in \mathcal{D}}\sum_{t,i}\epsilon
_{ti}\sum_{k}W_{tk}\left\langle d_{k},x_{ti}\right\rangle \\
&\leq &\frac{2}{nT}\sup_{D\in \mathcal{D}}\left( \sum_{k}\left\Vert
d_{k}\right\Vert ^{2}\right) ^{1/2}\max_{W\in \ext\left( \mathcal{W}_{\wedge
}\right) }\mathbb{E}\left( \sum_{k}\left\Vert \sum_{t,i}W_{tk}\epsilon
_{ti}x_{ti}\right\Vert ^{2}\right) ^{1/2} \\
&\leq &\frac{2\sqrt{K}}{nT}\max_{W\in \ext\left( \mathcal{W}_{\wedge }\right)
}\left( \sum_{t,i}\left( \sum_{k}W_{tk}^{2}\right) \left\Vert
x_{ti}\right\Vert ^{2}\right) ^{1/2} \leq \frac{2}{nT}\sqrt{K\sum_{t,i}\left\Vert x_{ti}\right\Vert ^{2}}=2\sqrt{%
\frac{K~{\rm tr}\left( \hat{C}\left( \mathbf{x}\right) \right) }{nT}},
\end{eqnarray*}%
where $\hat{C}\left( \mathbf{x}\right) $ is the total covariance operator
for all the data accross all tasks. Observe that we used no special
properties of the extreme points, in fact we only used $\left\Vert
W_{t}\right\Vert _{2}\leq 1$ for all $W\in \mathcal{W}$. For the weak
parameter we find 
\begin{eqnarray*}
\mathfrak{W}^{2} &=&\max_{W\in \ext\left( \mathcal{W}_{\wedge }\right)
}\sup_{D\in \mathcal{D}}\frac{1}{nT}\sum_{t,i}\left(
\sum_{k}W_{tk}\left\langle d_{k},x_{ti}\right\rangle \right) ^{2}=\max_{%
\mathbf{\phi \in }\left\{ 1,\dots,K\right\} ^{T}}\sup_{D\in \mathcal{D}}\frac{1%
}{nT}\sum_{t,i}\left\langle d_{\phi _{t}},x_{ti}\right\rangle ^{2} \\
&\leq &\frac{1}{T}\sum_{t}\sup_{\left\Vert d\right\Vert \leq 1}\frac{1}{n}%
\sum_{i}\left\langle d,x_{ti}\right\rangle ^{2}=\frac{1}{T}\sum_{t}\lambda
_{\max }\left( \hat{C}\left( \mathbf{x}_{t}\right) \right) ,
\end{eqnarray*}%
where $\hat{C}\left( \mathbf{x}_{t}\right) $ is the covariance of the data
of task $t$. The bound is%
\begin{equation*}
\mathcal{R}\left( \tciFourier \left( \mathcal{W}_{\wedge }\right) \circ 
\mathcal{G},\mathbf{x}\right) \leq 2\sqrt{\frac{K~{\rm tr}\left( \hat{C}\left( 
\mathbf{x}\right) \right) }{nT}}+8\sqrt{\frac{\left( 1/T\right)
\sum_{t}\lambda _{\max }\left( \hat{C}\left( \mathbf{x}_{t}\right) \right)
\ln \left( 2K\right) }{n}}.
\end{equation*}%
This result has already been announced in \citep{MPR2013}.

\subsection{Dictionary Learning with the Sharing Norm\label{section sharing
norm}}

We reverse the order of summation and maximum in the definition of the
previous norm to obtain the sharing norm%
\begin{equation*}
\left\Vert W\right\Vert _{\vee }=\sum_{k}\max_{t}\left\vert
W_{tk}\right\vert .
\end{equation*}%
This norm (under the name $1/\infty $ norm) has been applied to multitask
learning by various authors \citep[see e.g.][]{Liu 2009}. Statistical
guarantees in the form of oracle inequalities for multivariate regression
have been given by \citet{Negahban 2008}. None of these studies consider
dictionary learning. To apply our method we first observe that the
extreme points of the unit ball $\mathcal{W}_{\vee }\mathcal{=}\left\{
W:\left\Vert W\right\Vert _{\vee }\leq 1\right\} $ are now of the form $%
W_{tk}=v_{t}\delta _{k^{\ast },k}$ for some $\mathbf{v}\in \left\{
-1,1\right\} ^{T}$ and some $k^{\ast }\in \left\{ 1,\dots,K\right\} $. We have 
$\left\vert \ext\left( \mathcal{W}_{\vee }\right) \right\vert =2^{T}K$.

For the strong parameter we find, using (\ref{Trace inequality}),%
\begin{eqnarray*}
\mathfrak{S} &=&\frac{2}{nT}\max_{W\in \ext\left( \mathcal{W}_{\vee }\right) }%
\mathbb{E}\sup_{D\in \mathcal{D}}\sum_{t,i}\epsilon
_{ti}\sum_{k}W_{tk}\left\langle d_{k},x_{ti}\right\rangle =\frac{2}{nT}\max_{%
\mathbf{v},k^{\ast }}\mathbb{E}\sup_{D\in \mathcal{D}}\sum_{t,i}\epsilon
_{ti}v_{t}\left\langle d_{k^{\ast }},x_{ti}\right\rangle \\
&=&\frac{2}{nT}\mathbb{E}\left\Vert \sum_{t,i}\epsilon _{ti}x_{ti}\right\Vert
\leq \sqrt{\frac{{\rm tr}\left( \hat{C}\left( \mathbf{x}\right) \right) }{nT}}.
\end{eqnarray*}%
Here $v_{t}$ disappears in the third identity. It is absorbed by the
Rademacher variables, because the maximization is outside the expectation.
By the same token the supremum over the dictionary becomes a supremum over a
single vector $v$ with $\left\Vert v\right\Vert \leq 1$ which leads to
the norm. For the weak parameter we find%
\begin{eqnarray*}
\mathfrak{W}^{2} &=&\max_{W\in \ext\left( \mathcal{W}_{\wedge }\right)
}\sup_{D\in \mathcal{D}}\frac{1}{nT}\sum_{t,i}\left(
\sum_{k}W_{tk}\left\langle d_{k},x_{ti}\right\rangle \right)
^{2}=\sup_{\left\Vert d\right\Vert \leq 1}\frac{1}{nT}\sum_{t,i}\left\langle
d,x_{ti}\right\rangle ^{2} =\lambda _{\max }\left( \hat{C}\left( \mathbf{x}\right) \right) .
\end{eqnarray*}%
The overall bound is thus%
\begin{equation*}
\mathcal{R}\left( \tciFourier \left( \mathcal{W}_{\wedge }\right) \circ 
\mathcal{G},\mathbf{x}\right) \leq 2\sqrt{\frac{{\rm tr}\left( \hat{C}\left( 
\mathbf{x}\right) \right) }{nT}}+8\sqrt{\frac{\lambda _{\max }\left( \hat{C}%
\left( \mathbf{x}\right) \right) }{n}\left( \ln 2+\frac{\ln K}{T}\right) }.
\end{equation*}%
It depends only very weakly on the number $K$ of dictionary atoms and only
in the second term. Also observe that the weak parameter is never larger
than in case of the sparsity norm $\left\Vert .\right\Vert _{\wedge }$,
because $\hat{C}\left( \mathbf{x}\right) = 1/T \sum_{t}\hat{C}%
\left( \mathbf{x}_{t}\right) $ and $\lambda _{\max }\left( \cdot\right) $ is
convex on the cone of positive semidefinite operators.

A disadvantage of the sharing norm as a penalty is, that it makes strong
assumptions on the relatedness of the tasks in question, and that it is
sensitive to outlier tasks.

\subsection{Subspace Learning\label{section subspace learning}}

The final norm considered is%
\begin{equation*}
\left\Vert W\right\Vert _{S}=\max_{t}\left( \sum_{k}W_{tk}^{2}\right) ^{1/2},
\end{equation*}%
which provides an opportunity to demonstrate our method when the norm on $W$
is not polyhedral. We let $\mathcal{W}_{S}$ be the unit ball in $\left\Vert
.\right\Vert _{S}$ and require the dictionary to be orthonormal. This is the
class of multitask subspace learning \citep{Zhang 2005}, where the effective
weight vectors, $v_{t}=\sum_{k}W_{tk}d_{k}$, are constrained to all lie in a
subspace of dimension $K$ and to have norm bounded by one. We will derive a
bound which compares well with bounds derived from the much more advanced
methods of noncommutative Bernstein inequalities \citep{MP2012b}.

To apply our trick we first construct a finite approximation of $\mathcal{W}%
_{S}$ with the help of covering numbers. Let $\eta >0$. By \citep[][Prop. 5]{Cucker Smale 2001} we can find a subset $\mathcal{W}_{0}\subset 
\mathcal{W}_{S}$ such that $\forall W\in \mathcal{W}_{S}$, $\exists V\in 
\mathcal{W}_{0}$ such that $\left\Vert W-V\right\Vert _{S}\leq \eta $ and $%
\left\vert \mathcal{W}_{0}\right\vert \leq \left( 4/\eta \right) ^{KT}$. For
every $V\in \mathcal{W}_{0}$ let $\mathcal{W}_{V}=\left\{ W\in \mathcal{W}%
_{S}:\left\Vert W-V\right\Vert _{S}\leq \eta \right\} ,$ so that 
\begin{equation*}
\mathcal{W}_{S}=\bigcup_{V\in \mathcal{W}_{0}}\mathcal{W}_{V}\text{.}
\end{equation*}%
We apply Lemma \ref{Lemma Main}. By orthonormality of the dictionary the
weak parameter is 
\begin{equation*}
\mathfrak{W}^{2}=\max_{W\in \mathcal{W}_{S}}\sup_{D\in \mathcal{D}}\frac{1}{%
nT}\sum_{t,i}\left\langle \sum_{k}W_{tk}d_{k},x_{ti}\right\rangle
^{2}=\max_{\left\Vert v\right\Vert \leq 1}\frac{1}{nT}\sum_{t,i}\left\langle
v,x_{ti}\right\rangle ^{2}=\lambda _{\max }\left( \hat{C}\left( \mathbf{x}%
\right) \right) .
\end{equation*}%
The strong parameter can be bounded by two terms,%
\begin{eqnarray*}
\mathfrak{S} &=&\frac{2}{nT}\max_{V\in \mathcal{W}_{0}}\mathbb{E}\sup_{W\in 
\mathcal{W}_{V}}\sup_{D\in \mathcal{D}}\sum_{t,i}\epsilon
_{ti}\sum_{k}W_{tk}\left\langle d_{k},x_{ti}\right\rangle \\
&\leq &\frac{2}{nT}\max_{V\in \mathcal{W}_{0}}\mathbb{E}\sup_{D\in \mathcal{D%
}}\sum_{t,i}\epsilon _{ti}\sum_{k}V_{tk}\left\langle
d_{k},x_{ti}\right\rangle +\frac{2}{nT}\mathbb{E}\sup_{\left\Vert
W\right\Vert _{S}<\eta }\sup_{D\in \mathcal{D}}\sum_{t,i}\epsilon
_{ti}\sum_{k}W_{tk}\left\langle d_{k},x_{ti}\right\rangle .
\end{eqnarray*}%
The first term is bounded by $2\sqrt{K~{\rm tr}\left( \hat{C}\left( \mathbf{x}%
\right) \right) /\left( nT\right) }$ exactly as in Section \ref{Section
sparsity norm}. For the second term we again use orthonormality of the
dictionary 
\begin{align*}
& \frac{2}{nT}\mathbb{E}\sup_{\left\Vert W\right\Vert _{S}<\eta
}\sup_{D\in \mathcal{D}}\sum_{t}\left\langle
\sum_{k}W_{tk}d_{k},\sum_{i}\epsilon _{ti}x_{ti}\right\rangle \\
& =\frac{2}{nT}\mathbb{E}\sup_{D\in \mathcal{D}}\sum_{t}\sup_{\left\Vert
w\right\Vert \leq \eta }\left\langle \sum_{k}w_{k}d_{k},\sum_{i}\epsilon
_{ti}x_{ti}\right\rangle =\frac{2\eta }{nT}\sum_{t}\mathbb{E}\left\Vert
\sum_{i}\epsilon _{ti}x_{ti}\right\Vert \\
& \leq \frac{2\eta }{T}\sum_{t}\sqrt{\frac{{\rm tr}\left( \hat{C}\left( \mathbf{x}%
_{t}\right) \right) }{n}}\leq 2\eta \sqrt{\frac{\left( 1/T\right)
\sum_{t}{\rm tr}\left( \hat{C}\left( \mathbf{x}_{t}\right) \right) }{n}} =2\eta \sqrt{\frac{{\rm tr}\left( \hat{C}\left( \mathbf{x}\right) \right) }{n}},
\end{align*}%
where we used (\ref{Trace inequality}) and Jensen's inequality. Putting
everything together and taking the infimum over $\eta $ we get the bound%
\begin{eqnarray*}
\mathcal{R}\left( \tciFourier \left( \mathcal{W}_{S}\right) \circ 
\mathcal{G},\mathbf{x}\right) &\leq &2\sqrt{\frac{K~{\rm tr}\left( \hat{C}\left( 
\mathbf{x}\right) \right) }{nT}}+ \\
&&+\inf_{\eta >0}\left( 2\eta \sqrt{\frac{{\rm tr}\left( \hat{C}\left( \mathbf{x}%
\right) \right) }{n}}+8\sqrt{\frac{K\lambda _{\max }\left( \hat{C}\left( 
\mathbf{x}\right) \right) \ln \left( 4/\eta \right) }{n}}\right) .
\end{eqnarray*}%
If $\HH=%
\mathbb{R}
^{d}$ we may for example set $\eta =\sqrt{K/d}$ to obtain %
\begin{equation*}
\mathcal{R}\left( \tciFourier \left( \mathcal{W}_{S }\right) \circ 
\mathcal{G},\mathbf{x}\right) \leq 2\sqrt{\frac{K~{\rm tr}\left( \hat{C}\left( 
\mathbf{x}\right) \right) }{nT}}+8\sqrt{\frac{K\lambda _{\max }\left( \hat{C}%
\left( \mathbf{x}\right) \right) \ln \left( 16d/K\right) }{n}}.
\end{equation*}%

This can be compared to the bound derived from the results on trace norm
regularization in \citep{MP2012b}. The present bound gives a faster approach
to the limit as $T\rightarrow \infty $, but a larger limit value.

\subsection{The limit $T\rightarrow \infty $ in High Dimensions}

If $X_{1},\dots,X_{n}$ are sampled iid with $\| X_{i}\| \leq 1
$ then \citep[see e.g.][Theorem 7]{MP2012b}
$$
\mathbb{E}\sqrt{\mathbb{\lambda }_{\max }\left( \hat{C}\left( 
\mathbf{X}\right) \right) }\leq \sqrt{\lambda _{\max }\left( C\left(
X_{1}\right) \right) }+4
\sqrt{\frac{\ln \min(\dim(\HH),n) +1}{n}}.
$$ 
Here $C\left( X_{1}\right) $ is the true covariance $C\left(
X_{1}\right) =\mathbb{E}\hat{C}\left( X_{1}\right) $. This allows to
re-express our results in terms of expected Rademacher complexities, for
which bounds as in Theorem 1 exist \citep{Bartlett 2002}.

Now consider multitask dictionary learning as in the last two examples, with
data sampled from the uniform distribution on the unit sphere $\mathcal{S}%
^{d-1}$ in $%
\mathbb{R}
^{d}$. Even if our multitask model is appropriate, to achieve empirical
error $\eta $ we will likely need to work with margins of order $\eta /\sqrt{%
d}$ which incurs a Lipschitz constant of $\sqrt{d}/\eta $. On the other hand 
$C\left( X_{1}\right) $ has trace $1$ and largest eigenvalue $1/d$. Thus, for fixed $n$, as $T\rightarrow \infty$,
\[
\mathbb{E}\sqrt{\mathbb{\lambda }_{\max }\left( \hat{C}\left( \mathbf{X}%
\right) \right) }\rightarrow \sqrt{\frac{1}{d}}\]%
which cancels the contribution of the Lipschitz constant. Applied to the
bound in Section \ref{section sharing norm}, the ambient dimension disappears
completely in this limit. For subspace learning as in the previous section
it appears only in the logarithm. This unveils a mechanism how multitask
learning can potentially overcome the curse of high dimensionality.

\subsection*{Acknowledgments}

This work was supported in part by EPSRC Grant EP/H027203/1 and Royal Society International Joint Project 2012/R2.

\newpage

\section{Appendix}

For the reader's convenience we provide a self-contained proof of Theorem %
\ref{Lemma concentration of norm} with all the required intermediate
results. Most of the material of this appendix can also be found in the book 
by \citet{Boucheron 2013}.

\subsection{Concentration Inequalities and Proof of Theorem \protect\ref%
{Lemma concentration of norm}}

Let $\left( \Omega ,\Sigma ,\mu \right) =\prod_{i=1}^{n}\left( \Omega
_{i},\Sigma _{i},\mu _{i}\right) $ be a product of probability spaces. For
an event $\mathcal{E}\in \Sigma $ we write $\Pr \left( \mathcal{E}\right)
=\mu \left( \mathcal{E}\right) $. We denote a generic member of $\Omega $ by 
$\mathbf{x}=\left( x_{1},\dots,x_{n}\right) $. For $\mathbf{x}\in \Omega $, $%
1\leq k\leq n$ and $y\in \Omega $ we use $\mathbf{x}_{k\leftarrow y}$ to
denote the object obtained from $\mathbf{x}$ by replacing the $k$-th
coordinate of $\mathbf{x}$ with $y$. That is 
\begin{equation*}
\mathbf{x}_{k\leftarrow y}=\left( x_{1},\dots ,x_{k-1},y,x_{k+1},\dots
,x_{n}\right) \text{.}
\end{equation*}%
For $g\in L_{\infty }\left[ \mu \right] $ we write $\mathbb{E}g$ for $%
\int_{\Omega }g\left( \mathbf{x}\right) d\mu \left( \mathbf{x}\right) $, and
for $k\in \left\{ 1,\dots,n\right\} $ we introduce the functions $\mathbb{E}%
_{k}\left[ g\right] $, $\inf_{k}\left[ g\right] $ and $\sup_{k}\left[ g%
\right] $ by 
\begin{gather*}
\mathbb{E}_{k}\left[ g\right] \left( \mathbf{x}\right) =\int_{\Omega
_{k}}g\left( \mathbf{x}_{k\leftarrow y}\right) d\mu _{k}\left( y\right) , \\
\inf_{k}g=\inf_{y\in \Omega }g\left( \mathbf{x}_{k\leftarrow y}\right) \text{
and }\sup_{k}g=\sup_{y\in \Omega }g\left( \mathbf{x}_{k\leftarrow y}\right) ,
\end{gather*}%
where $\inf $ and $\sup $ on the r.h.s. are essential infima and suprema.
The functions $\mathbb{E}_{k}\left[ g\right] $, $\inf_{k}\left[ g\right] $
and $\sup_{k}\left[ g\right] $ are in $L_{\infty }\left[ \mu \right] $ and
do depend on $\mathbf{x}$ but not on $x_{k}$. Note that $\mathbb{E}_{k}\left[
g\right] $ correponds to the expectation conditional to all variables except 
$x_{k}$. We use $\left\Vert \cdot \right\Vert _{\infty }$ to denote the norm in $%
L_{\infty }\left[ \mu \right] $.

We will use and establish the following concentration inequalities:

\begin{theorem}
\label{Theorem Concentration}Let $F\in L_{\infty }\left[ \mu \right] $ and
define functionals $A$ and $B$ by%
\begin{eqnarray*}
A^{2}\left( F\right)  &=&\left\Vert \sum_{k=1}^{n}\left(
\sup_{k}F-\inf_{k}F\right) ^{2}\right\Vert _{\infty } \\
B^{2}\left( F\right)  &=&\left\Vert \sum_{k=1}^{n}\left( F-\inf_{k}F\right)
^{2}\right\Vert _{\infty }
\end{eqnarray*}%
Then for any $s>0$
\begin{eqnarray}
\nonumber
&{\rm (i)}&~\Pr \left\{ F>{{\mathbb{E}}}F+s\right\} \leq e^{-2s^{2}/A^{2}}\\
&{\rm (ii)}&~\Pr \left\{ F>{{\mathbb{E}}}F+s\right\} \leq e^{-s^{2}/\left(
2B^{2}\right) }.\nonumber
\end{eqnarray}
\end{theorem}

Part (i) is given in \citep{McDiarmid 1998}. The inequality (ii) appears in
different forms in various places \cite{Boucheron 2003,Ledoux 2001}. 
The constant $2$ in the exponent appears first in \citep{Maurer 2006}. We will use part (i) of the theorem to prove the following
Gaussian concentration inequality, also known as the
Tsirelson-Ibragimov-Sudakov inequality \citep{Ledoux 1991,Boucheron
2013}.

\begin{theorem}
\label{Theorem Gaussian Lipschitz}Let $F:\mathcal{%
\mathbb{R}
}^{n}\rightarrow 
\mathbb{R}
$ be Lipschitz with Lipschitz constant $L$ and let $\mathbf{X}=\left(
X_{1},\dots ,X_{n}\right) $ be a vector of independent random variables $%
X_{i}\sim \mathcal{N}\left( 0,1\right) $. Then for any $s>0$%
\begin{equation*}
\Pr \left\{ F>{{\mathbb{E}}}F+s\right\} \leq e^{-s^{2}/\left( 2L^{2}\right)
}.
\end{equation*}
\end{theorem}

Before proving these results we show how they can be used to obtain Theorem %
\ref{Lemma concentration of norm}.

\vspace{.3truecm}
\begin{proof}{\bf of Theorem \protect\ref{Lemma concentration of norm}.} We first consider the bounded case and denote $F\left( \mathbf{\epsilon }%
\right) =\sup_{\mathbf{z}\in A}~\left\langle \mathbf{\epsilon },\mathbf{z}%
\right\rangle $. For any given \textbf{$\epsilon $} let $\mathbf{z}\left( 
\mathbf{\epsilon }\right) \in A$ denote a corresponding maximizer in the
definition of $F$, so that $F\left( \mathbf{\epsilon }\right) =\left\langle 
\mathbf{\epsilon },z\left( \mathbf{\epsilon }\right) \right\rangle $. Now fix a configuration \textbf{$\epsilon $}. For any $j\in \left\{
1,\dots,n\right\} $ and $\eta \in \left[ -1,1\right] $ recall that \textbf{$%
\epsilon $}$_{j\leftarrow \eta }$ denotes the configuration \textbf{$%
\epsilon $} with $\epsilon _{j}$ replaced by $\eta $. Then for given $j$ and 
$\eta ^{\ast }$ minimizing $F\left( \mathbf{\epsilon }_{j\leftarrow \eta
}\right) $ we have%
\begin{eqnarray*}
F\left( \mathbf{\epsilon }\right) -\inf_{\eta \in \left[ -1,1\right]
}F\left( \mathbf{\epsilon }_{j\leftarrow \eta }\right)  &=&\left\langle 
\mathbf{\epsilon },z\left( \mathbf{\epsilon }\right) \right\rangle
-\left\langle \mathbf{\epsilon }_{j\leftarrow \eta ^{\ast }},z\left( \mathbf{%
\epsilon }_{j\leftarrow \eta ^{\ast }}\right) \right\rangle  \\
&\leq &\left\langle \mathbf{\epsilon },z\left( \mathbf{\epsilon }\right)
\right\rangle -\left\langle \mathbf{\epsilon }_{j\leftarrow \eta ^{\ast
}},z\left( \mathbf{\epsilon }\right) \right\rangle  =\left( \epsilon _{j}-\eta ^{\ast }\right) \mathbf{z}\left( \mathbf{%
\epsilon }\right) _{j}\leq 2\left\vert \mathbf{z}\left( \mathbf{\epsilon }%
\right) _{j}\right\vert .
\end{eqnarray*}%
It follows that 
\begin{equation*}
\sum_{j}\left( F\left( \mathbf{\epsilon }\right) -\inf_{\eta \in \left[ -1,1%
\right] }F\left( \mathbf{\epsilon }_{j\leftarrow \epsilon }\right) \right)
^{2}\leq 4\left\Vert \mathbf{z}\left( \mathbf{\epsilon }\right) \right\Vert
^{2}\leq 4\sup_{\mathbf{z}\in A}\left\Vert \mathbf{z}\right\Vert ^{2}
\end{equation*}%
and the conclusion follows from Theorem \ref{Theorem Concentration} (ii).

For the normal case observe that the function $\mathbf{x}\in 
\mathbb{R}
^{n}\mapsto \sup_{\mathbf{z}\in A}~\left\langle \mathbf{x},\mathbf{z}%
\right\rangle $ has Lipschitz constant $\sup_{\mathbf{z}\in A}\left\Vert 
\mathbf{z}\right\Vert $ and use Theorem \ref{Theorem Gaussian Lipschitz}%
. 
\end{proof}

\subsection{A General Concentration Result}

The proof of Theorems \ref{Theorem Concentration} and \ref{Theorem Gaussian
Lipschitz} is based on the entropy method \citep{Ledoux 2001,Boucheron 2003,Boucheron 2013}. We first establish the following
subadditivity property of entropy \citep{Ledoux 2001}.

\begin{theorem}
\label{Lemma Tensorization Inequality}Suppose $g:\Omega \rightarrow \mathbb{R%
}$ is positive. Then%
\begin{equation}
\mathbb{E}\left[ g\ln g\right] -\mathbb{E}\left[ g\right] \ln \mathbb{E}%
\left[ g\right] \leq \mathbb{E}\left[ \sum_{k=1}^{n}\left( \mathbb{E}_{k}%
\left[ g\ln g\right] -\mathbb{E}_{k}\left[ g\right] \ln \mathbb{E}_{k}\left[
g\right] \right) \right] .  \label{Tensorization Inequality}
\end{equation}
\end{theorem}

To prove this we use the following lemma.

\begin{lemma}
\label{Lemma Convexity of KL divergence}Let $h,g>0$ be bounded measurable
functions on $\Omega $. Then for any expectation $\mathbb{E}$%
\begin{equation*}
\mathbb{E}\left[ h\right] \ln \frac{\mathbb{E}\left[ h\right] }{\mathbb{E}%
\left[ g\right] }\leq \mathbb{E}\left[ h\ln \frac{h}{g}\right] .
\end{equation*}
\end{lemma}

\begin{proof}
Define an expectation functional $\mathbb{E}_{g}$ by $\mathbb{E}_{g}\left[ h%
\right] =\mathbb{E}\left[ gh\right] /\mathbb{E}\left[ g\right] $. The
function $\Phi \left( t\right) =t\ln t$ is convex for positive $t$, since $%
\Phi ^{\prime \prime }=1/t>0$. Thus, by Jensen's inequality,%
\begin{equation*}
\mathbb{E}\left[ h\right] \ln \frac{\mathbb{E}\left[ h\right] }{\mathbb{E}%
\left[ g\right] }=\mathbb{E}\left[ g\right] \Phi \left( \mathbb{E}_{g}\left[ 
\frac{h}{g}\right] \right) \leq \mathbb{E}\left[ g\right] \mathbb{E}_{g}%
\left[ \Phi \left( \frac{h}{g}\right) \right] =\mathbb{E}\left[ h\ln \frac{h%
}{g}\right] .
\end{equation*}
\end{proof}

\begin{proof}{\bf of Theorem \protect\ref{Lemma Tensorization Inequality}.} Write $g/\mathbb{E}\left[ g\right] $ as a telescopic product and use the
previous lemma to get%
\begin{eqnarray*}
\mathbb{E}\left[ g\ln \frac{g}{\mathbb{E}\left[ g\right] }\right]  &=&%
\mathbb{E}\left[ g\ln \prod_{k=1}^{n}\frac{\mathbb{E}_{1}\dots\mathbb{E}_{k-1}%
\left[ g\right] }{\mathbb{E}_{1}\dots\mathbb{E}_{k-1}\mathbb{E}_{k}\left[ g%
\right] }\right]  \\
&=&\sum_{k}\mathbb{E}\left[ \mathbb{E}_{1}\dots\mathbb{E}_{k-1}\left[ g\right]
\ln \frac{\mathbb{E}_{1}\dots\mathbb{E}_{k-1}\left[ g\right] }{\mathbb{E}%
_{1}\dots\mathbb{E}_{k-1}\left[ \mathbb{E}_{k}\left[ g\right] \right] }\right] 
\\
&\leq &\sum_{k}\mathbb{E}\left[ g\ln \frac{g}{\mathbb{E}_{k}\left[ g\right] }%
\right] =\mathbb{E}\left[ \sum_{k}\mathbb{E}_{k}\left[ g\ln \frac{g}{\mathbb{%
E}_{k}\left[ g\right] }\right] \right].
\end{eqnarray*}
\end{proof}

Fix some $F\in L_{\infty }\left[ \mu \right] $. For any real $\beta $ and $%
g\in L_{\infty }\left[ \mu \right] $ define the thermal expectation $\mathbb{%
E}_{\beta F}\left[ g\right] $\ and for $1\leq k\leq n$ the conditional
thermal expectation $\mathbb{E}_{k,\beta F}\left[ g\right] $ by%
\begin{equation*}
\mathbb{E}_{\beta F}\left[ g\right] =\frac{\mathbb{E}\left[ ge^{\beta F}%
\right] }{\mathbb{E}\left[ e^{\beta F}\right] }\text{ and }\mathbb{E}%
_{k,\beta F}\left[ g\right] =\frac{\mathbb{E}_{k}\left[ ge^{\beta F}\right] 
}{\mathbb{E}_{k}\left[ e^{\beta F}\right] }.
\end{equation*}%
Also let $\sigma _{\beta F}^{2}\left[ g\right] $ and $\sigma _{k,\beta F}^{2}%
\left[ g\right] $ be the corresponding variances%
\begin{equation*}
\sigma _{\beta F}^{2}\left[ g\right] =\mathbb{E}_{\beta F}\left[ g^{2}\right]
-\left( \mathbb{E}_{\beta F}\left[ g\right] \right) ^{2}\text{ and }\sigma
_{k,\beta F}^{2}\left[ g\right] =\mathbb{E}_{k,\beta F}\left[ g^{2}\right]
-\left( \mathbb{E}_{k,\beta F}\left[ g\right] \right) ^{2}.
\end{equation*}%
Note that $\mathbb{E}_{k,\beta F}\left[ g\right] $ and $\sigma _{k,\beta
F}^{2}\left[ g\right] $ depend on $\mathbf{x}$ but not on of $x_{k}$. Also $%
\mathbb{E}_{k,\beta F}\left[ g\right] =\mathbb{E}_{k,\beta \left( F+h\right)
}\left[ g\right] $ for any function $h$ which does not depend on $x_{k}$.
The Helmholtz free energy and its conditional counterpart are for $\beta \neq
0$%
\begin{equation*}
H\left( \beta \right) =\frac{1}{\beta }\ln \mathbb{E}\left[ e^{\beta F}%
\right] \text{ and }H_{k}\left( \beta \right) =\frac{1}{\beta }\ln \mathbb{E}%
_{k}\left[ e^{\beta F}\right] \text{.}
\end{equation*}%
Here we omit the dependence on $F$. Note that $\lim_{\beta \rightarrow
0}H\left( \beta \right) =\mathbb{E}\left[ F\right] $.

\begin{lemma}
\label{Lemma double integral}We have 
\begin{equation*}
H^{\prime }\left( \beta \right) =\frac{1}{\beta ^{2}}\int_{0}^{\beta
}\int_{t}^{\beta }\sigma _{sF}^{2}\left[ F\right] dsdt,~~\text{\rm and }~%
H_{k}^{\prime }\left( \beta \right) =\frac{1}{\beta ^{2}}\int_{0}^{\beta
}\int_{t}^{\beta }\sigma _{k,sF}^{2}\left[ F\right] dsdt
\end{equation*}
\end{lemma}

\begin{proof}
Define a function $A$ by $A\left( \beta \right) =\ln \mathbb{E}\left[
e^{\beta F}\right] $. Then $A\left( 0\right) =0$. It is easy to verify that $%
A^{\prime }\left( \beta \right) =\mathbb{E}_{\beta F}\left[ F\right] $ and $%
A^{\prime \prime }\left( \beta \right) =\sigma _{\beta F}^{2}\left[ F\right] 
$. Thus%
\begin{equation*}
\mathbb{E}_{\beta F}\left[ F\right] =A^{\prime }\left( \beta \right)
=A^{\prime }\left( 0\right) +\int_{0}^{\beta }A^{\prime \prime }\left(
t\right) dt=\mathbb{E}\left[ F\right] +\frac{1}{\beta }\int_{0}^{\beta
}\int_{0}^{\beta }\sigma _{sF}^{2}\left[ F\right] dsdt
\end{equation*}%
and%
\begin{eqnarray*}
\ln \mathbb{E}\left[ e^{\beta F}\right]  &=&A\left( \beta \right)
=\int_{0}^{\beta }A^{\prime }\left( t\right) dt=\int_{0}^{\beta }\left(
A^{\prime }\left( 0\right) +\int_{0}^{t}A^{\prime \prime }\left( s\right)
ds\right) dt \\
&=&\beta \mathbb{E}\left[ F\right] +\int_{0}^{\beta }\int_{0}^{t}\sigma
_{sF}^{2}\left[ F\right] dsdt.
\end{eqnarray*}%
Thus%
\begin{equation*}
H^{\prime }\left( \beta \right) =\frac{1}{\beta }\mathbb{E}_{\beta F}\left[ F%
\right] -\frac{1}{\beta ^{2}}\ln \mathbb{E}\left[ e^{\beta F}\right] =\frac{1%
}{\beta ^{2}}\int_{0}^{\beta }\int_{t}^{\beta }\sigma _{sF}^{2}\left[ F%
\right] dsdt,
\end{equation*}%
which gives the first equation. The proof of the second is completely
analogous. 
\end{proof}

We now give a general concentration result \citep[see e.g.][]{Maurer 2012}.

\begin{theorem}
\label{Theorem general concentration}For any $\beta >0$ we have the entropy
bound%
\begin{equation}
H^{\prime }\left( \beta \right) \leq \frac{1}{\beta ^{2}}\mathbb{E}_{\beta F}%
\left[ \sum_{k=1}^{n}\int_{0}^{\beta }\int_{t}^{\beta }\sigma _{k,sF}^{2}%
\left[ F\right] dsdt\right]   \label{General Entropy Bound}
\end{equation}%
and with $t>0$ the concentration inequality%
\begin{equation}
\Pr \left\{ F-\mathbb{E}F>t\right\} \leq \exp \left( \beta \int_{0}^{\beta
}H^{\prime }\left( \gamma \right) d\gamma -\beta t\right) .
\label{General Concentration Bound}
\end{equation}
\end{theorem}

\begin{proof}
Substituting $g=e^{\beta F}$ in (\ref{Tensorization Inequality}), dividing
by $\beta ^{2}\mathbb{E}\left[ e^{\beta F}\right] $ and using Lemma \ref%
{Lemma double integral} we arrive at%
\begin{eqnarray*}
H^{\prime }\left( \beta \right)  &\leq &\mathbb{E}_{\beta F}\left[
\sum_{k}\left( \frac{1}{\beta }\mathbb{E}_{k,\beta F}\left[ F\right] -\frac{1%
}{\beta ^{2}}\ln \mathbb{E}_{k}\left[ e^{\beta F}\right] \right) \right] =%
\mathbb{E}_{\beta F}\left[ \sum_{k}H_{k}^{\prime }\left( \beta \right) %
\right] . \\
&=&\frac{1}{\beta ^{2}}\mathbb{E}_{\beta F}\left[ \sum_{k}\int_{0}^{\beta
}\int_{t}^{\beta }\sigma _{k,sF}^{2}\left[ F\right] dsdt\right] ,
\end{eqnarray*}%
which is the first conclusion. Integrating $H^{\prime }$ from $0$ to $\beta $%
, using $\lim_{\beta \rightarrow 0}H\left( \beta \right) =\mathbb{E}\left[ F%
\right] $, and multiplying with $\beta $ gives%
\begin{equation*}
\ln \mathbb{E}\left[ e^{\beta F}\right] \leq \beta \mathbb{E}\left[ F\right]
+\beta \int_{0}^{\beta }H^{\prime }\left( \gamma \right) d\gamma .
\end{equation*}%
Subtract $\beta \left( \mathbb{E}\left[ F\right] -t\right) $ and take the
exponential to get%
\begin{equation*}
\mathbb{E}\left[ e^{\beta \left( F-\mathbb{E}\left[ F\right] -t\right) }%
\right] \leq \exp \left( \beta \int_{0}^{\beta }H^{\prime }\left( \gamma
\right) d\gamma -\beta t\right) .
\end{equation*}%
The second conclusion then follows from Markov's inequality.
\end{proof}

\subsection{Proofs of Theorems \protect\ref{Theorem Concentration} and 
\protect\ref{Theorem Gaussian Lipschitz}}

With Theorem \ref{Theorem general concentration} at hand we can prove a
number of concentration inequalities if we manage to bound the right hand
side in (\ref{General Entropy Bound}). We then substitute in the second
conclusion and optimize over $\beta $. At first the expression with the
double integral and the thermal variances looks very cumbersome, but, as we
shall see, it can often be bounded by comparatively simple methods. The
bounded difference inequality, Theorem \ref{Theorem Concentration} (i) is
obtained very easily, the proof of Theorem \ref{Theorem Concentration} (ii)
is slightly more tricky.

\vspace{.3truecm}
\begin{proof}{\bf of Theorem \protect\ref{Theorem Concentration}} 
We prove (i). For fixed $\mathbf{x}$ the thermal variance $\sigma _{k,sF}^{2}%
\left[ F\right] $ is the variance of a function with values in the interval $%
\left[ \inf_{k}F,\sup_{k}F\right] $, so that 
\begin{equation*}
\sigma _{k,sF}^{2}\left[ F\right] \leq \frac{1}{4}\left(
\sup_{k}F-\inf_{k}F\right) ^{2}.
\end{equation*}%
The double integral then just gives a factor of $\beta ^{2}/2$. Now sum over 
$k$ and bound the expectation $\mathbb{E}_{\beta F}$ by the $\left\Vert
.\right\Vert _{\infty }$-norm to obtain%
\begin{equation*}
H^{\prime }\left( \beta \right) \leq \frac{1}{8}\left\Vert \sum_{k}\left(
\sup_{k}F-\inf_{k}F\right) ^{2}\right\Vert _{\infty }=\frac{A^{2}\left(
F\right) }{8}.
\end{equation*}%
(\ref{General Concentration Bound}) then gives%
\begin{equation*}
\Pr \left\{ F-\mathbb{E}F>t\right\} \leq \exp \left( \frac{\beta ^{2}}{8}%
A^{2}\left( F\right) -\beta t\right) 
\end{equation*}%
and substitution of $\beta =4t/A^{2}\left( F\right) $ gives the result.

To prove part (ii) first note that for any expectation and any real function 
$g$ we have $\sigma ^{2}\left[ g\right] =\min_{t\in 
\mathbb{R}
}\mathbb{E}\left[ \left( g-t\right) ^{2}\right] \leq \mathbb{E}\left[ \left(
g-\inf g\right) ^{2}\right] $. Applied to the conditional thermal variance
this translates to 
\begin{equation}
\sigma _{k,\beta F}^{2}\left[ F\right] \leq \mathbb{E}_{k,\beta F}\left[
\left( F-\inf_{k}F\right) ^{2}\right] \text{.}
\label{Trivial variance bound}
\end{equation}%
We now claim that the right hand side above is a nondecreasing function of $%
\beta $. Too see this write $h=F-\inf_{k}F$ and define a real function $\xi $
by $\xi \left( t\right) =\left( \max \left\{ t,0\right\} \right) ^{2}$.
Since $h\geq 0$ we have 
\begin{equation*}
\mathbb{E}_{k,\beta F}\left[ \left( F-\inf_{k}F\right) ^{2}\right] =\mathbb{E%
}_{k,\beta \left( F-\inf_{k}F\right) }\left[ \left( F-\inf_{k}F\right) ^{2}%
\right] =\mathbb{E}_{k,\beta h}\left[ \xi \left( h\right) \right] .
\end{equation*}%
Here we used $\mathbb{E}_{k,\beta \left( F+h\right) }=\mathbb{E}_{k,\beta F}$
whenever $g$ is independent of $x_{k}$. A straighforward computation shows%
\begin{equation*}
\frac{d}{d\beta }\mathbb{E}_{\beta h}\left[ \xi \left( h\right) \right] =%
\mathbb{E}_{\beta h}\left[ \xi \left( h\right) h\right] -\mathbb{E}_{\beta h}%
\left[ \xi \left( h\right) \right] \mathbb{E}_{\beta h}\left[ h\right] \geq
0,
\end{equation*}%
where the last inequality uses the well known fact that for any expectation $%
\mathbb{E}\left[ \xi \left( h\right) h\right] \geq \mathbb{E}\left[ \xi
\left( h\right) \right] \mathbb{E}\left[ h\right] $ whenever $\xi $ is a
nondecreasing function. This establishes the claim.

Together with (\ref{Trivial variance bound}) this implies that for $0\leq
s\leq \beta $ we have%
\begin{equation*}
\sigma _{k,sF}^{2}\left[ F\right] \leq \mathbb{E}_{k,sF}\left[ \left(
F-\inf_{k}F\right) ^{2}\right] \leq \mathbb{E}_{k,\beta F}\left[ \left(
F-\inf_{k}F\right) ^{2}\right] ,
\end{equation*}%
so, using Theorem \ref{Theorem general concentration} again,%
\begin{eqnarray*}
H^{\prime }\left( \beta \right)  &\leq &\frac{1}{\beta ^{2}}\mathbb{E}%
_{\beta F}\left[ \sum_{k=1}^{n}\int_{0}^{\beta }\int_{t}^{\beta }\sigma
_{k,sF}^{2}\left[ F\right] ds~dt\right] \leq \frac{1}{2}\mathbb{E}_{\beta F}%
\left[ \sum_{k=1}^{n}\mathbb{E}_{k,\beta F}\left( F-\inf_{k}F\right) ^{2}%
\right]  \\
&=&\frac{1}{2}\mathbb{E}_{\beta F}\left[ \sum_{k=1}^{n}\left(
F-\inf_{k}F\right) ^{2}\right] \leq \frac{B^{2}\left( F\right) }{2},
\end{eqnarray*}%
where we used the identity $\mathbb{E}_{\beta F}\mathbb{E}_{k,\beta F}=%
\mathbb{E}_{\beta F}$. Then (\ref{General Concentration Bound}) gives 
\begin{equation*}
\Pr \left\{ F-\mathbb{E}F>t\right\} \leq \exp \left( \frac{\beta
^{2}B^{2}\left( F\right) }{2}-\beta t\right) 
\end{equation*}%
and substitution of $\beta =t/B^{2}\left( F\right) $ gives the
result. 
\end{proof}

Finally we use the bounded difference inequality, Theorem \ref{Theorem
Concentration} (i), to prove the Gaussian Concentration inequality.

\vspace{.3truecm}
\begin{proof}{\bf of Theorem \protect\ref{Theorem Gaussian Lipschitz}.} 
By an easy approximation argument using convolution with Gaussian kernels of
decreasing width it suffices to prove the result if the function $F$ is in $%
C^{\infty }$ with $\left\vert \left( \partial ^{2}/x_{i}^{2}\right) F\left( 
\mathbf{x}\right) \right\vert \leq B$ for all $\mathbf{x\in 
\mathbb{R}
}^{n}$ and $i\in \left\{ 1,\dots,n\right\} $, where $B$ is a finite, but
potentially very large, constant. For $K\in 
\mathbb{N}
$ let $X_{i}^{\left( K\right) }$ be the random variable%
\begin{equation*}
X_{i}^{\left( K\right) }=\frac{1}{\sqrt{K}}\sum_{k=1}^{K}\epsilon _{ik},
\end{equation*}%
where the $\epsilon _{ik}$ are independent Rademacher variables, and define
the random vector $\mathbf{X}^{\left( K\right) }$ accordingly. We write $%
G\left( \mathbf{\epsilon }\right) =F\left( \mathbf{X}^{\left( K\right)
}\right) $ and set about to apply Theorem \ref{Theorem Concentration} (i) to
the random variable $G\left( \mathbf{\epsilon }\right) $ by bounding the
variation in the epsilon components. 

Fix a configuration $\mathbf{\epsilon }$ with corresponding vector $\mathbf{X%
}^{\left( K\right) }$. For each $i\in \left\{ 1,\dots,n\right\} $ we introduce
the real function $F_{i}\left( x\right) =F\left( \mathbf{X}_{i\leftarrow
x}^{\left( K\right) }\right) $. Since $F$ is $C^{\infty }$ we have for any $%
t\in 
\mathbb{R}
$
\begin{equation*}
F_{i}\left( x+t\right) -F_{i}\left( x\right) =tF_{i}^{\prime }\left(
x\right) +\frac{t^{2}}{2}F_{i}^{\prime \prime }\left( s\right) 
\end{equation*}
for some $s\in 
\mathbb{R}
$, and by the Lipschitz condition and the bound on $\left\vert F_{i}^{\prime
\prime }\right\vert $ 
\begin{eqnarray*}
\left( F_{i}\left( x+t\right) -F_{i}\left( x\right) \right) ^{2}
&=&t^{2}\left( F_{i}^{\prime }\left( x\right) \right)
^{2}+t^{3}F_{i}^{\prime }\left( x\right) F_{i}^{\prime \prime }\left(
s\right) +\frac{t^{4}}{4}\left( F_{i}^{\prime \prime }\left( s\right)
\right) ^{2} \\
&\leq &t^{2}\left( F_{i}^{\prime }\left( x\right) \right) ^{2}+\left\vert
t\right\vert ^{3}LB+\frac{t^{4}}{4}B^{2}.
\end{eqnarray*}

Now fix a pair of indices $\left( i,j\right) $ with $i\in \left\{
1,\dots,n\right\} $ and $k\in \left\{ 1,\dots,K\right\} $. Since $\epsilon _{ik}$
can only have two values, one of which must be $X_{i}^{\left( K\right) }$,
we must have%
\begin{align*}
\left( \sup_{y}G\left( \mathbf{\epsilon }_{\left( i,k\right) \leftarrow
y}\right) -\inf_{y}G\left( \mathbf{\epsilon }_{\left( i,k\right) \leftarrow
y}\right) \right) ^{2}& =\left( F_{i}\left( X_{i}^{\left( K\right) }\pm 
\frac{2}{\sqrt{K}}\right) -F_{i}\left( X_{i}^{\left( K\right) }\right)
\right) ^{2} \\
& \leq \frac{4\left( F_{i}^{\prime }\left( X_{i}^{K}\right) \right) ^{2}}{K}+%
\frac{8LB}{K^{3/2}}+\frac{4B^{2}}{K^{2}}.
\end{align*}%
Summing over $k$ and $i$ and then taking the supremum over $\mathbf{\epsilon 
}$ we obtain%
\begin{equation*}
A\left( G\right) ^{2}\leq 4L^{2}+\frac{8nLB}{K^{1/2}}+\frac{4nB^{2}}{K}.
\end{equation*}%
From Theorem \ref{Theorem Concentration} (i) and $F\left( \mathbf{X}^{\left(
K\right) }\right) =G\left( \mathbf{\epsilon }\right) $ we conclude that%
\begin{equation*}
\Pr \left\{ F\left( \mathbf{X}^{\left( K\right) }\right) -{{\mathbb{E}}}
F\left( \mathbf{X}^{\left( K\right) }\right) >s\right\} \leq \exp \left( 
\frac{-s^{2}}{2L^{2}+4nLB/K^{1/2}+2nB^{2}/K}\right) .
\end{equation*}%
The conclusion now follows from the central limit theorem since $\mathbf{X}%
^{\left( K\right) }\rightarrow \mathbf{X}$ weakly as $K\rightarrow \infty $.
 
\end{proof}

\end{document}